\documentclass[a4paper]{article}

\usepackage[a4paper,top=3cm,bottom=2cm,left=3cm,right=3cm,marginparwidth=1.75cm]{geometry}
\usepackage[T1]{fontenc}   \usepackage{algorithm}
\usepackage{algorithmic}
\usepackage{microtype}
\usepackage{amsmath,amssymb,latexsym,epsfig,amsthm}
\usepackage{graphicx}
\usepackage{color}
\usepackage{caption}
\usepackage{subcaption}
\usepackage{url}
\usepackage{paralist}
\usepackage{enumitem}
\usepackage{booktabs} \usepackage{natbib}
\usepackage{mathtools}
\usepackage{tabularx}
\usepackage{textcomp}
\usepackage{authblk}

\usepackage{hyperref}
\hypersetup{
    colorlinks=true,
    citecolor=blue
}

\DeclarePairedDelimiterX{\infdivx}[2]{(}{)}{#1\;\delimsize\|\;#2}
\newcommand{\infdiv}{\cD\infdivx}
\newcommand{\infdivcvar}{\cD_{\text{CVaR}}\infdivx}
\newcommand{\divcvar}{\cD_{\text{CVaR}}}

\global\long\def\reals{\mathbf{R}}

\global\long\def\E{\mathbb{E}}

\global\long\def\U{\mathcal{U}}
\global\long\def\Udiv{\mathcal{U}^{\text{div}}}
\global\long\def\Uint{\mathcal{U}^{\text{div}\cap\text{intv}}}
\global\long\def\Uintv{\mathcal{U}^{\text{intv}}}

\global\long\def\R{\mathcal{R}}
\global\long\def\V{V}
\global\long\def\X{X}

\global\long\def\ouralg{\texttt{RISe}}
\global\long\def\divalg{\texttt{JointDRO}}
\global\long\def\intvalg{\texttt{Causal}}
\global\long\def\erm{\texttt{ERM}}
\global\long\def\stdalg{\texttt{Standard}}

\DeclareMathOperator*{\argmin}{arg\,min}

\newcommand\nonumberthis{\nonumber\refstepcounter{equation}}

\newtheorem{assumption}{Assumption}
\newtheorem{proposition}{Proposition}

\def\to{{\,\rightarrow\,}}

\mathchardef\mhyphen="2D

\newcommand{\vertiii}[1]{{\left\vert\kern-0.25ex\left\vert\kern-0.25ex\left\vert #1
    \right\vert\kern-0.25ex\right\vert\kern-0.25ex\right\vert}}

\def\cA{\mathcal{A}}

\def\cD{\mathcal{D}}

\def\cN{\mathcal{N}}

\def\cR{\mathcal{R}}
\def\cS{\mathcal{S}}

\def\cU{\mathcal{U}}

\begin{document}

\title{Learning Under Adversarial and Interventional Shifts}

\author[1]{Harvineet Singh\footnote{Work done as a summer fellow at the Center for Research on Computation and Society, Harvard University. Correspondence at \texttt{hs3673@nyu.edu}}}
\author[2]{Shalmali Joshi}
\author[2]{Finale Doshi-Velez}
\author[2]{Himabindu Lakkaraju}
\affil[1]{New York University}
\affil[2]{Harvard University}

\date{}
\maketitle

\begin{abstract}
Machine learning models are often trained on data from one distribution and deployed on others. So it becomes important to design models that are robust to distribution shifts. Most of the existing work focuses on optimizing for either adversarial shifts or interventional shifts. Adversarial methods lack expressivity in representing plausible shifts as they consider shifts to joint distributions in the data. Interventional methods allow more expressivity but provide robustness to unbounded shifts, resulting in overly conservative models. In this work, we combine the complementary strengths of the two approaches and propose a new formulation, RISe, for designing robust models against a set of distribution shifts that are at the intersection of adversarial and interventional shifts. We employ the distributionally robust optimization framework to optimize the resulting objective in both supervised and reinforcement learning settings. Extensive experimentation with synthetic and real world datasets from healthcare demonstrate the efficacy of the proposed approach.
\end{abstract}

\section{Introduction}

As machine learning (ML) models are increasingly being deployed in high-stakes applications, it is important to ensure the safety of these models. One way to ensure this desirable property is to train models robust to distribution shifts. This is particularly important for deployments in consequential decision-making such as in healthcare and criminal justice. Guaranteeing good predictive performance on (unseen) test distributions requires departing from choosing the best model for the train distribution. Two broad principles have been proposed for learning robust models \textit{without} access to samples from test distribution -- adversarial training~\citep{sinha2018certifying, staib2017distributionally, madry2018towards} and causal invariance~\citep{subbaswamy2019preventing, arjovsky2019invariant, pearl2011transportability}. Interestingly, recent work has shown correspondence between the two within the framework of \textit{distributionally robust optimization} (DRO) \citep{buhlmann2020, rojas2018invariant}.

In DRO based learning, the goal is to optimize a \textit{worst-case} risk over a set of distributions defined to be \textit{close} in some metric to a nominal distribution (typically, the train distribution). 
The set is referred to as the uncertainty set as it encodes our uncertainty about the test distribution. Under a large class of divergence metrics for defining such sets, one can give rigorous generalization guarantees owing to the worst-case nature of the optimization \citep{duchi2018learning, staib2019distributionally}. At the same time, if the uncertainty set is too large, the DRO solution can be too pessimistic. That is, minimizing for the worst-case risk will result in higher risk on the test distributions we actually might encounter in practice. 
For example, considering perturbations in images like adversarial noise or rotations can result in large sets which do not correspond to `realistic' shifts \citep{taori2020measuring}. Depending on the type of models and available sample sizes, robust learning can significantly degrade performance~\citep{raghunathan2020understanding}.

On the other hand, methods that employ causal invariance are highly expressive, allowing domain experts to precisely pinpoint specific shifts one may encounter in practice, thereby making them more realistic. However, even under the specified shifts, causal methods
result in models that are robust to unbounded shifts, a scenario unlikely to be encountered in practice.
Hence, the balance between utility (average-case performance) and robustness (worst-case performance) of the models depends critically on the framework, and thereby, on the definition of the uncertainty sets.
Prior works consider uncertainty sets either using divergence metrics or interventions on causal models, each of which suffer from above mentioned limitations.  
We are motivated by the insight that  carefully considering the cases where individual approaches to robust learning may be highly conservative or lack expressivity can help improve utility for realistic shifts. We propose to combine the two approaches to be able to express more realistic and bounded shifts, and thus, train less conservative methods. 

\paragraph{Our contributions.} 
\begin{asparaitem}
\item We consider the intersection of adversarial and causal robustness frameworks as a natural way to specify plausible uncertainty sets. These novel uncertainty sets characterize bounded interventions on causal models of the domain. We refer to our formulation as $\ouralg$ which stands for \underline{R}obustness with \underline{I}ntersection \underline{Se}ts.
\item We give an efficient procedure for solving the new DRO problem for such sets under certain conditions.
\item The procedure is applied to problems from supervised learning, contextual bandits (CB), and reinforcement learning (RL) on synthetic and real-world datasets. For CB and RL, we provide a novel formulation for robust off-policy evaluation (OPE) based on the proposed uncertainty sets.  
\end{asparaitem}
In summary, we provide a novel perspective to the problem of robust learning that bridges existing adversarial learning and causal inference methods.

\section{Related Work}
\paragraph{Generalization in ML.} Improving generalizability (to unseen test domains) of ML models has been studied using multiple principles. Adversarial training~\citep{madry2018towards,sinha2018certifying,staib2017distributionally} improves robustness by training models that are invariant to perturbations of training samples within a norm ball. Approaches based on causal invariance train predictors robust to unbounded shifts specified as interventions in (parts of) the underlying data generating mechanism.~\citet{subbaswamy2019preventing} identify stable estimators when such causal knowledge is available. When such domain knowledge is unavailable, causal methods seek invariant subsets (e.g. causal parents of the target variable) using multiple training environments~\citep{xie2020risk, krueger2020out, rosenfeld2020risks, koyama2020out}. There are several drawbacks of relying on one or the other framework. In contrast, we take a more practical approach to describing bounded shifts expressed as interventions in the causal model of the domain.

\paragraph{Distributionally Robust Optimization.} 
The primary mechanism of DRO is to specify uncertainty sets that encode the uncertainty about potential test shifts. These uncertainty sets can be defined over the joint distribution of the data~\citep{ben2013robust, bertsimas2018data, blanchet2019quantifying,duchi2018learning} or the marginal distribution of a subset of features~\citep{duchi2019distributionally}, and are particularly well explored for supervised learning. Interestingly, adversarial training can also be understood as solving DRO with a Wasserstein metric-based set \citep{staib2017distributionally}.
Applications of DRO have been explored in contextual bandits for policy learning~\citep{si2020distributional,mo2020learning,faury2020distributionally} and evaluation~\citep{kato2020off,jeong2020robust}. Uncertainty sets based on KL-divergence, $L_1,L_2$ and $L_\infty$ norms have been studied~\citep{nilim2005robust, iyengar2005robust} in robust RL. Such methods may lead to conservative value estimates if the norm ball-based uncertainty sets form a much larger set than the set of actual test distributions. Some frameworks aim to restrict the sets by iteratively refining them with newly observed data \citep{petrik2019beyond}, however, the sets are still constructed using $L_1$ norm balls for reasons of tractability of the optimization problem. In contrast, we use domain knowledge elicited with a causal graph to restrict the sets and leverage metrics, such as $f$-divergences, that allow us to maintain tractability and favourable statistical properties for the resulting optimization problem 
~\citep{duchi2018learning, shapiro2014lectures}. 

\paragraph{Causal Robustness.}
Causal approaches provide robustness to arbitrarily strong distributional shifts, as shown in recent works \citep{rojas2018invariant, subbaswamy2019preventing, magliacane2018domain, rothenhausler2018anchor, peters2016causal}. A relaxation to bounded shifts, caused by interventions that change the mean of the variables, has been proposed by \citet{rothenhausler2018anchor} for supervised learning in the special case of linear Gaussian models. 
Extensions to non-linear settings, but still within the restrictive class of additive shifts, have been explored ~\citep{christiansen2020causal}.
In contrast, our proposed sets are defined without parametric assumptions on shifts or variables.
\citet{subbaswamy2019universal} provide a hierarchy of distribution shifts specified on the causal graph. Further, they show the trade-off between utility and robustness while mitigating these shifts in the causal framework. However, they do not explore the combination of intervention and divergence-based shifts as a way to make this trade-off.

\section{Preliminaries}

\begin{figure}[htbp!]
\centering
    \begin{subfigure}[b]{0.4\textwidth}
    \centering
        \includegraphics[scale=0.6]{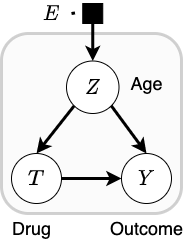}
        \caption{}
\end{subfigure} \begin{subfigure}[b]{0.4\textwidth}
    \centering
\includegraphics[width=\textwidth]{{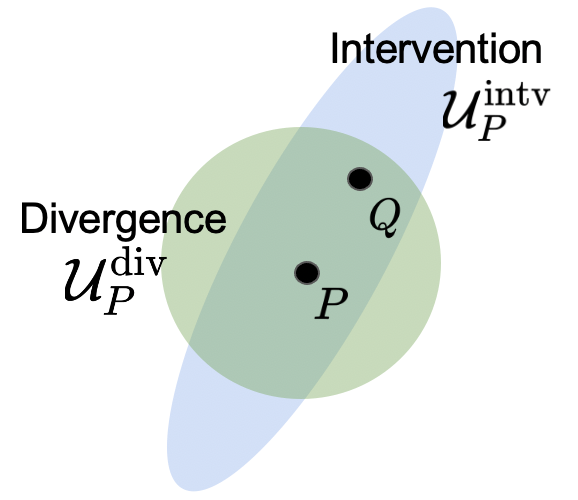}}
    \caption{}
\end{subfigure}
  \caption{(a) Example causal graph: Drugs are prescribed based on age and have different outcomes by age. Selection node $E$ (for environment) specifies a plausible distribution shift (posited by a domain expert), denoting that age distribution or $P(Z)$ may change across environments. (b) \ouralg: Intersection of the two uncertainty sets that reduces the size of the divergence-based set while still remaining within the class of plausible shifts given by the intervention-based set.}
  \label{fig:fig1}
\end{figure}

\textbf{Notation.} 
$\V$ denotes all observed variables in the domain. Features are denoted by $\X$ and outcome variable by $Y$. Train and test distributions over $\V$ are denoted by $P$ and $Q$ respectively (using the same notation for their densities). An uncertainty set w.r.t. a distribution $P$ is denoted by $\U_P$.

\textbf{Data generating process.} A causal graph $\mathcal{G}$ is a directed acyclic graph.
Each node is a variable of interest (here $\V$) with directed edges representing causal dependencies. Selection diagrams \citep{pearl2011transportability} augment the causal graph to represent plausible shifts across environments. It annotates the graph with additional nodes (called selection nodes) pointing to variables whose mechanisms may differ across train and test environments. Consider the example in Figure \ref{fig:fig1}(a) with two features $(Z,T)$ age and prescribed drug. The edge from selection node $E\rightarrow Z$ indicates that the age distribution shifts across environments (say, countries). Parents of a node $Y$ are denoted by $pa(Y)$. Interventions on a node, $do(Z)$, artificially manipulate the underlying mechanism generating $Z$~\citep{pearl2009causality}. 
For instance, a soft intervention, $do(Z \sim \nu(Z))$, manipulates the \textit{distribution} of $Z$ from $P(Z)$ to $\nu(Z)$. In the selection diagram, (soft) interventions on children of $E$ (denoted $ch(E)$) gives us all environments that confirm with the underlying data generating process, hence are plausible or realistic. 

\textbf{DRO objective.} In DRO, we want to find a decision variable $\theta$ (e.g. regression parameters) that minimizes the \textit{robust risk}, defined as worst-case loss across the uncertainty set,
\begin{equation}
\theta^*\in \argmin_{\theta\in \Theta} \left\{\R(\theta, \U_P) := \sup_{Q\in\ \U_P}\ \E_{\V\sim Q}[\ell(\theta, \V)]\right\}
     \label{eq:robust}
\end{equation}
where $\ell(\theta, \V)$ is a given loss function. This problem can be seen as a game between the min player, choosing $\theta$ to minimize expected loss for the test distribution $Q$, and the max player trying to find the worst $Q$ from the set $\U_P$.

\textbf{Divergence-based sets.} Here, the uncertainty set $\U_P$ is all distributions lying in a $\delta$-ball around the train distribution, defined with respect to a divergence metric $\cD$ as,
\begin{equation}
\label{eq:distance_set}
    \mathcal{U}^\text{div}_{P} := \{Q\ll P\ \text{s.t.}\ \infdiv{Q}{P}\leq \delta\}
\end{equation}
where $Q{\ll} P$ denotes absolute continuity i.e. $P(\V)=0$ implies $Q(\V)=0$.
An example of divergence-based sets are the ones based on $\infdivcvar{Q}{P} = \log{\sup_{X\in \texttt{dom}(X)}} \frac{Q(X)}{P(X)}$ where CVaR stands for Conditional Value at Risk \citep{rockafellar2000optimization}. 

\citet{duchi2019distributionally} consider learning robust classifiers under the assumption that only the covariate distribution $P(X)$ may change at test time and the conditional distribution $P(Y\vert X)$ remains the same. For this setup, they define uncertainty sets using $\divcvar$ and provide estimators for solving the robust risk.
We will leverage these estimators for solving DRO problems arising in our formulation.

\paragraph{Intervention-based sets.} Using causal knowledge provides an alternate approach to distributional robustness.
Interestingly, the causal solution can be shown to be the solution to a DRO problem with sets defined by arbitrary interventions on $\X$ \citep[Thm. 4]{rojas2018invariant},
\begin{equation}
\label{eq:causal_set}
\Uintv_P := \{Q\ll P\ \text{s.t.}\ Q=P(V\mid do(X\sim \nu(X)))\}
\end{equation}
The solution to Eq. (\ref{eq:robust}) with $\Uintv_P$ is to use the parents of $Y$ as features for learning the robust predictor \citep{buhlmann2020}. That is, for the squared loss, the DRO solution is the Bayes predictor $\E[Y\vert pa(Y)]$.\footnote{The result requires that the mechanism between $X$ and $Y$ does not change due to the intervention on $X$ \citep{buhlmann2020}.}
An important distinction from the divergence-based sets is allowing for \textit{arbitrary} shifts in $P(X)$ instead of a bounded shift in some metric.
We now demonstrate our approach to specify plausible shifts by considering the intersection of both types of sets.

\section{Our Approach -- Intersection of Divergence and Intervention-based Sets}
We propose to improve utility of robust methods by leveraging both the approaches such that we can specify realistic shifts using interventions (as opposed to divergences) with bounded magnitude (as opposed to arbitrary interventions). Defining the uncertainty set this way may help the modeller to achieve a better trade-off between utility and robustness.

In practice, the nature of distribution shifts across domains can be highly systematic. For instance, in our example, the age distribution of diabetics can shift systematically from a large city hospital to a small rural hospital. 
If the divergence based uncertainty sets considered are large (and over the whole feature set as opposed to specific features such as age in this case), the optimization problem in Eq.(\ref{eq:robust}) can result in degenerate solutions  \citep{hu2018does}.
Training a model that allows for arbitrary interventional shifts in the age distribution can further minimize utility of our models. Specifically, when we expect the train and test distributions to not vary much, then the robust loss under the intervention-based set overestimates the plausible loss that we will observe at test time. The advantage of the worst-case guarantee for the causal solution is further limited to cases where we can find the Bayes predictor in the considered function class with reasonable sample sizes. In the experiments, we illustrate how 
model misspecification can invalidate this approach's attractive guarantees to arbitrary shifts.

We demonstrate that considering intersections of the uncertainty sets associated with the two approaches can help describe realistic shifts and thus mitigate their individual drawbacks. 
Let $Z \subseteq \X$ be the subset of variables that shift across domains. In the selection diagram, $Z$ are the children of the selection node or $ch(E)$. The intersection of the uncertainty sets $\cU^{\text{div}}_P$ and $\Uintv_P$ contains all distributions resulting from interventions on selection nodes $Z$ which are bounded in some metric $\cD$. That is, \begin{equation}
\label{eq:int_set}
\begin{aligned}
    \Uint_P := \{&Q\ll P\ \text{s.t.}\ Q=P(V\mid do(Z\sim \nu(Z))), \\ &\infdiv{\nu(Z)}{P(Z)} \leq \delta, Z:=ch(E)\}.
\end{aligned}
\end{equation}

This definition in Eq. (\ref{eq:int_set}) limits the arbitrary shifts in Eq. (\ref{eq:causal_set}) by requiring bounded shifts. Thus, Eq. (\ref{eq:int_set}) is the intersection of sets (\ref{eq:distance_set}) and (\ref{eq:causal_set}), which reduces the size of the uncertainty set.
The notation assumes no other variable causes $Z$ for conciseness. If $Z$ has parent nodes, then the divergence $\cD$ can be defined for the conditional distribution of $Z\vert pa(Z)$ assuming its existence. 

\paragraph{Optimizing with intersection sets.} The optimization problem with intersection of uncertainty sets is the DRO problem of Eq.(\ref{eq:robust}) minimizing the worst-case risk, $\R(\theta, \Uint_P)$, which is now also a function of the \textit{interventional} distributions. Thus, identification of the causal estimand, worst-case risk in this case, requires making assumptions about the augmented causal graph. To make the problem tractable, we make the following assumptions.

\begin{assumption}
\label{assum:inter}
For the causal graph $\mathcal{G}$ with the selection node $E$, assume the following holds.
\begin{enumerate}[label=(\alph*), leftmargin=*]
    \item $\mathcal{G}$ is a directed acyclic graph with no unmeasured confounding, i.e. $P(V){=}\prod_{O\in V} P(O\vert pa(O))$,
    \item $ch(E)$ have no parents except $E$, i.e. shifts only occur in variables with no parents.
\end{enumerate}
\end{assumption}

While Assumption \ref{assum:inter}\emph{(a)} is a statement about thoroughness of the measured variables and is frequently made in domains such as epidemiology \citep{hernan2020causal}, Assumption \ref{assum:inter}(b) stems from the tractability of solving the minmax problem in (\ref{eq:robust}). While, evaluating risk under bounded distributional shifts in conditional distribution have been recently addressed in~\citet{subbaswamy2020evaluating}, we defer addressing empirical evaluation for bounded conditional interventional shifts to future work. Our assumptions still cover many important shifts as demonstrated in experiments. The following result allows us to efficiently solve the DRO problem over the specified intersection set.

\begin{proposition}\label{sec:prop1}
Given Assumption \ref{assum:inter} holds, the uncertainty set for the intersection (\ref{eq:int_set}) is equivalent to,
\begin{equation}
\label{eq:marginal_set}
\begin{aligned}
    \Uint_P = \{&Q\ll P\ \text{s.t.}\ Q=\nu(Z)P(V\setminus Z\mid Z), \\ &\infdiv{\nu(Z)}{P(Z)}\leq \delta, Z:=ch(E)\}
\end{aligned}
\end{equation}
\end{proposition}
That is, the proposed intersection-based uncertainty set consists of distributions with bounded shifts in the marginals of some feature set $Z$.  The proof is a straight-forward application of factorization from Assumption \ref{assum:inter}(a), followed by replacing factors $P(Z)$ with the intervened distribution $\nu(Z)$ by the definition of the $do$-intervention. The proof is provided in Appendix~\ref{app:prop1}.

Next, we apply $\ouralg$ in three settings: \begin{inparaenum} \item[i)] Supervised Learning, \item[ii)] OPE in Contextual Bandits and \item[iii)] OPE in MDPs. \end{inparaenum}
\subsection{Robust Supervised Learning with \ouralg.}
\label{sec:superlearn_intersect}
Here, $\V{:=}(\X,Y)$ are i.i.d. feature and outcome pairs, and $\theta$ is the prediction function.
We want to solve the minmax problem (\ref{eq:robust}) with the intersection set $\Uint_P$.
Using Proposition~\ref{sec:prop1}, $\Uint_P$ is equivalent to
the bounded shifts in subset $Z=ch(E)$, Eq. (\ref{eq:marginal_set}). As this expression does not contain any interventional (causal) terms, the loss under intersection sets is identified using observations from the train distribution alone. 

Choosing $\cD=\divcvar$ in Eq. (\ref{eq:marginal_set}), the minmax problem is the same as that considered by \citet{duchi2019distributionally}. Thus, we use their  marginal DRO framework which upper bounds the worst-case risk
using convex duality and assuming smoothness of the risk conditioned on $Z$.
In essence, the problem reduces to computing the worst risk among all subpopulations of the training set of at least a certain size, controlled by $\delta$ (see Appendix~\ref{app:irslset} for a justification).
We note that the marginal DRO framework \citep{duchi2019distributionally} solves a \emph{specific instance} of the (bounded) interventional shift, i.e. with $\divcvar$, and only for supervised learning. In the next two sections, we study OPE with intersection-based uncertainty sets which presents a novel extension of this framework.

\subsection{Robust OPE in Contextual Bandits}
\label{sec:offpolicy}
Evaluating a policy from batch data under shifts is critical for applications like healthcare where the new policies need to be vetted before deployment \citep{li2020optimizing}. 
Here we have access to $n$ tuples $\{(Z_i, T_i, Y_i)\}_i$ collected with a known stochastic policy that applies treatment $T_i$ in context $Z_i$ and observes the corresponding outcome $Y_i$. 
Further, assume that the tuples are i.i.d. as in the contextual bandit setup, where the joint distribution factorizes as {\small{$\prod_i P(Z_i)P(T_i\vert Z_i)P(Y_i|Z_i,T_i)$}}.

Given data sampled from $P$, the goal in off-policy evaluation (OPE) is to evaluate the expected outcome $\E_Q[Y]$ under a distribution $Q$ induced by following a new policy in the \textit{same} environment \citep{dudik2014doubly, thomas2016data}. Importantly, the difference between the two distributions is assumed to be only due to different policies i.e. $P(T\vert Z){\neq} Q(T\vert Z)$ and the rest of the environment-related factors are the same. That is, in terms of the causal graph for this setup, shown in Figure \ref{fig:fig1}(a), only shift interventions on $T$ are considered \citep{zhang2017transfer}.

\paragraph{Robust OPE in CBs with \ouralg.}
Departing from this assumption, we evaluate a policy under a \textit{new} environment characterized by unknown and bounded interventions on $Z$. Thus, we define the uncertainty set $\cU^{\text{CB}}_P$ by
shifts in the marginal distribution of $Z$:
\begin{equation}
\begin{aligned}
    \U^{\text{CB}}_P = \{&Q\ll P\ \text{s.t.}\ Q=\nu(Z)Q(T\vert Z)P(Y\vert Z, T), \\ &
    \infdiv{\nu(Z)}{P(Z)}\leq \delta\}
\end{aligned}
\end{equation}
In the robust OPE problem, the aim is to find the \textit{worst-case} average outcome under $\U^{\text{CB}}_P$ instead of just the average,
\begin{equation}
\label{eq:robust_eval}
\R(\cU^{\text{CB}}_P) = \inf_{Q\in\ \cU^{\text{CB}}_P}\ \E_{(Z, T, Y)\sim Q}[Y],
\end{equation}
where $Q(T\vert Z)$ is the policy to be evaluated and is considered to be known and fixed. Consider each distribution in the set $\cU^{\text{CB}}_P$, $Q(\cdot){=}\nu(Z)Q(T\vert Z)P(Y\vert Z, T)$, which differs from the train distribution in the factors for $Z$ and $T\vert Z$. 
\ifdefined\gobble \vspace{-.5cm}\fi
\begin{align}\label{eq:robust_ope_cb}
\R(\U^\text{CB}_P) &= \inf_{Q\in \U^\text{CB}_P}\ \E_{Z\sim Q(Z)}\E_{P(Y|T,Z)Q(T | Z)}[Y] \\ 
    &\stackrel{(\ref{eq:robust_ope_cb}a)}{=} \inf_{Q \in \U^\text{CB}_P} \E_{Z\sim Q(Z)}\E_{P(Y|T,Z)P(T | Z)}\left[\frac{Q(T|Z)}{P(T|Z)}Y\right] \nonumberthis \\
    &\stackrel{(\ref{eq:robust_ope_cb}b)}{=}\inf_{\eta\in \reals} \frac{1}{\delta}\E_{Z \sim P(Z)}\left[(\E\left[W{\times} Y\vert Z\right]-\eta)_+\right] + \eta \nonumberthis
\end{align}
To solve Eq. (\ref{eq:robust_eval}) for this $Q$, we first use importance sampling to account for the change in $T\vert Z$ due to the known policy $Q(T\vert Z)$, Eq.~(\ref{eq:robust_ope_cb}a). As a result, the set $\U^{\text{CB}}_P$ now consists of shifts on $Z$ alone, which is the same as the set $\Uint_P$ in Eq. (\ref{eq:marginal_set}). Thus, the Robust OPE problem reduces to solving $\sup_{Q\in\ \Uint_P}\ \E_{\V\sim Q}[W{\times} Y]$ where $W$ are the importance sampling weights, $W(T, Z){=}Q(T\vert Z)/P(T\vert Z)$. 
Using similar convex duality arguments as in Appendix~\ref{app:rslduchi} and~\citet{shapiro2014lectures}, and mapping the problem to that of marginal subpopulation shifts, as done in the supervised learning case, we have Eq.~(\ref{eq:robust_ope_cb}b). 
The full optimization procedure is summarized in Algorithm~\ref{alg:cb_ope} in Appendix~\ref{app:cb_ope}.

\subsection{Robust OPE in MDPs}
\label{sec:fullmdp}
We now consider shifts in transition dynamics across train-test environments which can invalidate OPE methods for MDPs as they typically assume stationary dynamics. 
\begin{figure}[htbp!]
\centering
    \begin{subfigure}[b]{0.6\linewidth}
        \includegraphics[scale=0.45]{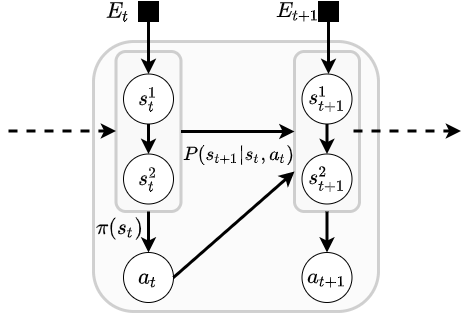}
        \caption{}
    \end{subfigure}\begin{subfigure}[b]{0.3\linewidth}
        \includegraphics[width=0.65\textwidth]{{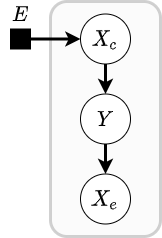}}
        \caption{}
    \end{subfigure}
\caption{(a) \textbf{Causal graph for MDP.}
    Based on the selection node $E$, only $P(s^1\vert s,a)$ may change from one environment to another. (b) \textbf{Causal graph for supervised learning.} Only $P(X_c)$ changes in test environments.}
    \label{fig:sel_mdp}
\end{figure}
Assume a finite state, finite action infinite-horizon MDP with discount factor $\gamma\in [0,1)$, states {$s {\in} \cS$}, actions {$a {\in} \cA$}, reward model {$r(s,a)$}, and transition model {$P(\cdot\vert s,a)$}. For simplicity, we assume that the reward model is fixed across environments, while the transition model may vary. As we will use interventions to define shifted environments, we assume that a structural causal model augments the MDP (see \citep{buesing2018woulda} for a formal definition).
We have a batch of trajectories obtained using a policy $\mu : \cS \to \Delta(\cA)$.
Our goal is to evaluate \textit{robust} value for a given policy $\pi$, in a new environment with unknown transition dynamics. Hence, the uncertainty set for each state-action pair is defined over $P(s'|s, a)$, 
denoted by $\mathcal{U}(s,a)$.\footnote{We drop the dependence on the train environment's transition probabilities for conciseness and use $P$ to denote target probabilities, instead of $Q$, to avoid confusion with the $Q$-value function.} Specifically, we want to estimate the robust value for $\pi$ starting from $s_0$ as $V^\pi(s_0) = \inf_{P\in\mathcal{U}(s,a)} \E_{\pi,P}\left[\sum_{t=0}^\infty \gamma^t r(s_t,a_t)\vert s_0\right]$.
~\citet{iyengar2005robust} proves that $V^\pi(\cdot)$ is the solution to the following fixed-point equation (namely, robust Bellman equation) if we assume that uncertainty sets for each state-action pair are constructed independently (known as \textit{SA-rectangularity}),
\begin{equation}
    \label{eq:robust_bellman}
    \begin{aligned}
        V^\pi(s) = r(s,\pi(s)) + \inf_{P\in\mathcal{U}(s,\pi(s))} \gamma \E_{s'\sim P(s'\vert s,\pi(s))}[V^\pi(s')]
\end{aligned}
\end{equation}
This can be solved iteratively, i.e. by dynamic programming~\citep{sutton2018reinforcement}. Given the value function at any iteration, we additionally have to solve 
the minimization problem over $\mathcal{U}(s,a)$.
Thus, the robust OPE in MDPs reduces to solving multiple DRO problems where we have different choices for the uncertainty set.

\paragraph{Robust OPE in MDPs with \ouralg.}
Past work has only considered divergence-based sets \citep{iyengar2005robust,tamar2014scaling,petrik2019beyond}.
In contrast, we consider sets based on the shifts specified in the causal graph for the MDP, Figure \ref{fig:sel_mdp}.
At any timestep $t$, the states are partitioned into two sets - direct children of the selection node and the rest.
Suppose, $(s^{1}_t,s^{2}_t)$ are the $2$ feature sets of the state vector $s_t$. We deduce that $P(s^{1}_{t+1}\vert s_t, a_t)$ shifts across environments but $P(s^{2}_{t+1}\vert s_t, a_t, s^{1}_{t+1})$ is the same as that in the train environment (which we denote by $P_{0}(s^{2}_{t+1}\vert s_t, a_t)$). Thus, the uncertainty set needs to be defined only for $P(s^{1}_{t+1}\vert s_t, a_t)$, denoted by $\U^{\text{MDP}}(s,a)$:
\begin{align*}
     \U^{\text{MDP}}(s,a) := \{&P(\cdot\vert s,a)\ll P_0(\cdot\vert s,a)\ \text{s.t.} \\
     &\forall s', P(s'\vert s,a)=\nu(s'^{1}\vert s, a)P_0(s'^{2}\vert s, a, s'^{1}), \\ &\infdiv{\nu(s'^{1}\vert s, a)}{P_0(s'^{1}\vert s, a)}\leq \delta\}
\end{align*}
With $\U^{\text{MDP}}$, the DRO subproblem in Eq.~(\ref{eq:robust_bellman}) is reduced to,
\begin{align*}
    &\inf_{P\in\mathcal{U}^{\text{MDP}}(s,\pi(s))} \E_P[V^\pi(s')]\\
    &= \inf_{P\in\mathcal{U}^{\text{MDP}}(s,\pi(s))} \E_{P(s'^{1}\vert s, \pi(s))} \left[\E_{P_0(s'^{2}\vert s, \pi(s), s'^{1})}\left[V^\pi(s')\right]\right]\\
    &=: \cR_{\mathcal{U}^{\text{MDP}}(s, \pi(s))} \left(V^\pi(s')\right)
\end{align*}
We estimate the inner expectation with Monte-Carlo averaging on batch data, followed by solving the DRO problem as done in bandits, Eq.~(\ref{eq:robust_ope_cb}). Here, we use the maximum likelihood estimate of the transition model $P_0$ as we do not have access to the true model (see details in Appendix \ref{app:algmdp}).
 
\section{Experiments and Results}\label{sec:expts}
We comprehensively evaluate the proposed sets in simulation studies with synthetic and real data for the three problems described in Sections \ref{sec:superlearn_intersect}, \ref{sec:offpolicy}, and \ref{sec:fullmdp}. 
In Section \ref{sec:exp_superlearn}, we consider covariate shifts in supervised learning and demonstrate the loss in utility with divergence sets as compared to intersection sets. Additionally, we show the drawbacks of the causal approach. Next, in Section \ref{sec:exp_bandit}, we consider the contextual bandits setup and compare methods that account for partial feedback and covariate shifts.
We show that our method provides more faithful estimates for efficacy of a drug dosing policy under patient population shifts. In Section \ref{sec:exp_fullmdp}, we consider sequential decision problems under shifts in environment dynamics. Within a simulated Sepsis environment, we show that the proposed intersection sets provide significantly less conservative value estimates for a treatment policy.
Experiment setup, baselines, and results are discussed separately for each problem.

\subsection{Robust Supervised Learning}
\label{sec:exp_superlearn}
The data generating process for this synthetic example is modelled after~\citep{duchi2019distributionally}.
There are two covariates, a cause $X_c$ and an effect $X_e$ of the outcome $Y$. 
The structural equations are as follows,
\begin{align*}
    E&\sim \text{Bernoulli}(\{0,1\},\delta_0),
    X_c\sim \begin{cases}
        -4 + \cN(0,2), \text{if}\ E=0\\\
        4 + \cN(0,2), \text{if}\ E=1
    \end{cases}\\
    Y&\sim \begin{cases}
    0.125 X_c^2 + \cN(0,0.1), & \text{if}\ X_c\leq 0\\
    0.125 X_c^2 + \cN(0,2), & \text{if}\ X_c > 0
  \end{cases},\\
  X_e&\sim 0.125 Y^2 + \cN(0,8).
\end{align*}
There are two groups in the dataset, $X_c\geq 0$ and $X_c< 0$, with different outcome functions. Selection node $E$ controls the relative proportion of each group. In training data ($n=2000$), the group with $X_c< 0$ are in minority. The test data ($n=2000$) has increasing proportion of minority group as $\delta_0$ varies from $0.2$ to $0.9$. 

\begin{figure}[htbp!]
\centering
\includegraphics[scale=0.35]{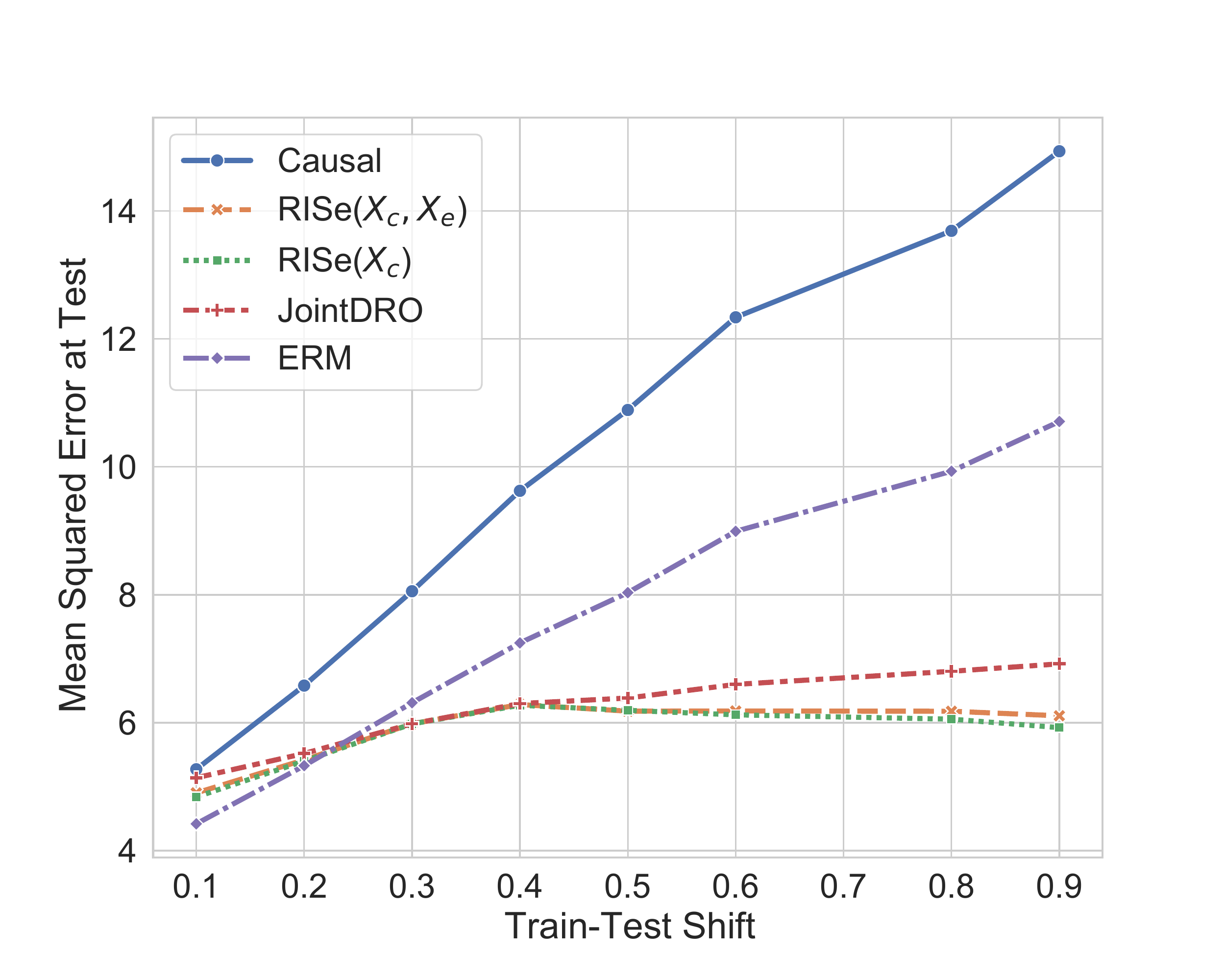}
\caption{\textbf{Robust supervised learning under covariate shifts.} MSE on test sets (y-axis) with varying levels of shift in $X_c$ (x-axis) is shown. 
We observe that solving DRO with bounded shifts in $X_c$ alone results in a model that performs well across different shift magnitudes.}
\label{fig:example_dro}
\end{figure}We compare with empirical risk minimization ($\erm$) $\E[Y\vert X_c, X_e]$, causal solution $\E[Y\vert pa(Y)] \equiv \E[Y\vert X_c]$, and $\divalg$ \citep{duchi2018learning} which considers bounded shifts on \textit{all} variables $Y, X_c, X_e$. 
Since the shift is only on $X_c$, our approach fits models which are robust to bounded shifts in $X_c$ using the estimator (\ref{eq:marginal_smooth_app}), \ouralg$(X_c)$. We also include our method run with shifts in both $X_c,X_e$, namely \ouralg$(X_c,X_e)$.
For all methods, we fit linear models with square loss and no bias term. We choose a misspecified model to illustrate the effect of misspecification on the generalization of the causal solution compared to other methods. In this graph, $\erm$ ($\E[Y\vert X_c, X_e]$) can also be seen as a causal approach given by another robust learning method, called \textit{surgery estimator} \citep{subbaswamy2019preventing}. All hyperparameters including learning rate \texttt{lr}, Lipschitz constant \texttt{L}, and training-time robustness level $\delta$ are optimized separately for each test environment and for each method (details are in Appendix~\ref{app:expts_sl}). In practice, we cannot tune the methods for each test environment as we only have access to one. We do so in this synthetic setup to evaluate the best versions of each method.

Figure \ref{fig:example_dro} shows the MSE for each method on test sets with increasing minority group proportion. Note than $\erm$ performs the best when test-train shifts are low, as expected. However, $\erm$ is not robust as test error increases with more shift. Among the DRO methods, we observe that $\divalg$ is overly conservative, i.e. has high error, as it aims to be robust to shifts in all variables. It is followed by $\ouralg(X_c,X_e)$, and then $\ouralg(X_c)$ as we successively decrease the uncertainty set sizes while still capturing the actual shifts.
Although the causal solution ($\intvalg$) should be robust to shifts on $X_c$, but due to model misspecification, the error is higher than the other approaches and increases as minority group proportion increases. Thus, the proposed approach achieves low test error for both small and large shifts, whereas other approaches are either overly conservative ($\divalg$) or are not robust to model misspecification ($\intvalg$).

\begin{figure}[tbp!]
\centering
\includegraphics[width=0.6\linewidth]{{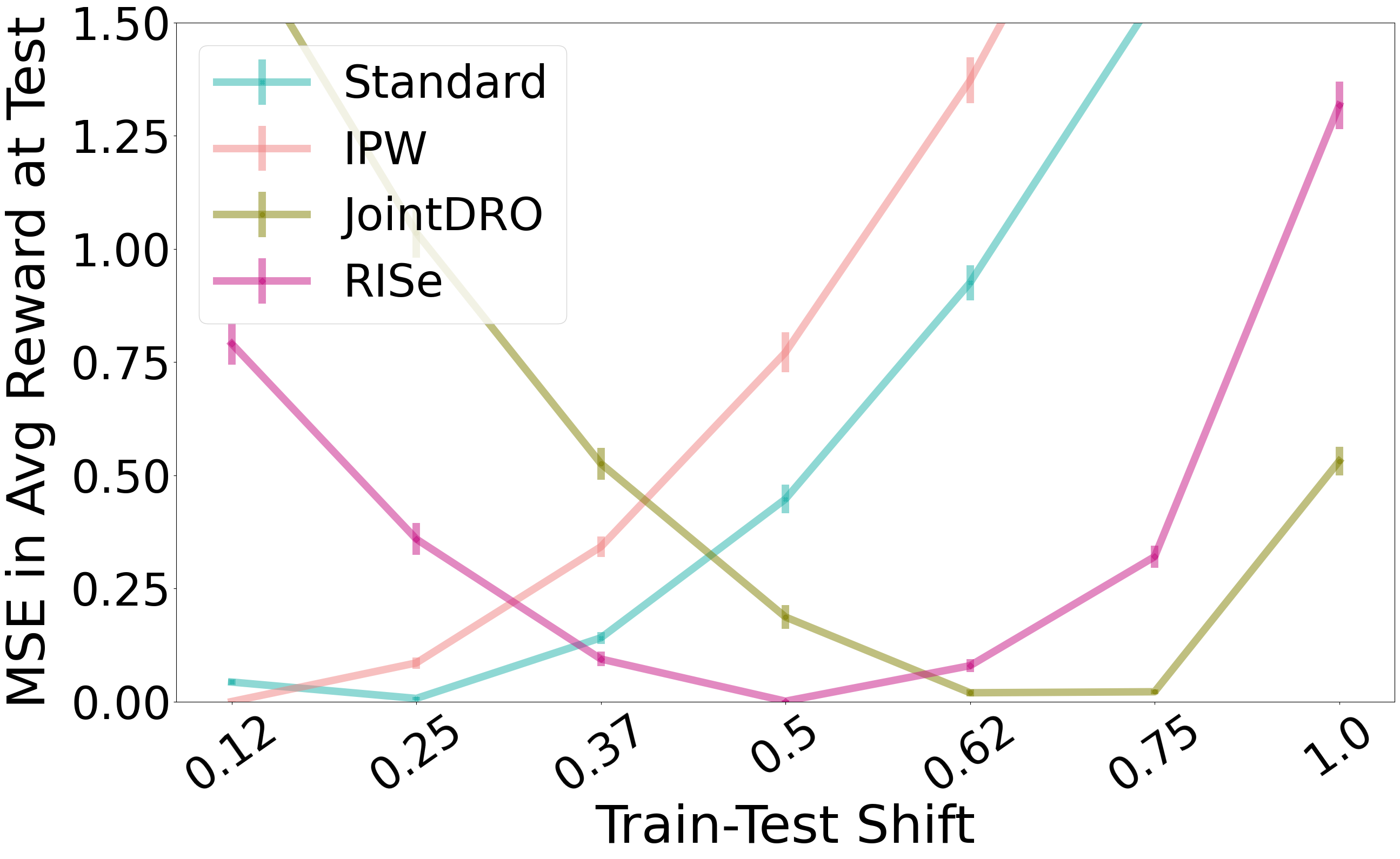}}
\caption{\textbf{Robust OPE in the contextual bandit setting in Figure \ref{fig:fig1}(a).} MSE in the value estimate for test environments (y-axis) with varying levels of shift in $Z_1$ (x-axis) is shown. Robustness level in (\ref{eq:int_set}) is set to $\delta=0.8$. Our approach performs well for moderate shifts. Mean and error bars are computed over $5$ random initializations.}
\label{fig:synth_cb}
\end{figure}
\subsection{Robust OPE in Bandits}
\label{sec:exp_bandit}

\paragraph{Synthetic data}
We generate data according to the causal graph in Figure \ref{fig:fig1}(a) with two features $Z:=(Z_1, Z_2)$ in the context, binary treatment $T$ and continuous outcome $Y$. Structural equations are:
\begin{align*}
& E\sim {\small{\text{Bernoulli}}}(-1,1,\delta_0),Z_1\sim E\cdot \cN(10,1), 
 Z_2\sim \cN(5,1)\\
    &Y_{T=t}\sim \cN(Z^\top w_t,{0.1}^2), t\in\{0,1\}\\
    & \text{where } w_0=[0.1,0.1], w_1=[0.1,0.5] \\
&T\sim \text{Bernoulli}(\sigma(Z^\top\beta + \beta_0)), \beta=[0.1,0.1]
\end{align*}
where $\sigma(z)=1/(1+e^{-z})$. The bias term in the train policy is $\beta_0=-1$, whereas the policy to be evaluated in test environments has $\beta_0=-0.5$. In addition, the marginal distribution of $Z_1$ is changed, via change in $E$, by increasing $\delta_0$ in the test environment. We simulate $n=2000$ samples in the train environment and use it for all methods. The objective is to estimate the robust value of the new policy. 

We compare the proposed method with the following baselines:
$\stdalg$ returns the average value assuming no shift.
{\small{\texttt{Inverse Probability Weighting (IPW)}}} only corrects for shift in policy using importance sampling. $\divalg$ accounts for shifts in all variables $V$ and is therefore more conservative than the proposed method.~\citet{si2020distributional} recently proposed $\divalg$ using $f$-divergences for contextual bandits .

In Figure \ref{fig:synth_cb}, we plot the MSE between the estimated and the true policy value, evaluated using $n=20000$ samples from the test environment. We observe that when the test environment is close to the train one, not accounting for the shift ({$\stdalg$}) performs well. But, as the shift increase, our approach does better. With large shifts, larger uncertainty sets are required. Thus, {$\divalg$} does better than the other methods. In summary, the proposed approach performs well when the shift is significant but not too large. This highlights the importance of choosing the desired robustness level appropriately which is a challenging problem for DRO methods in general.
\begin{figure}[tbp!]
\centering
\begin{subfigure}[t]{0.8\textwidth}
  \centering
  \includegraphics[width=0.8\textwidth]{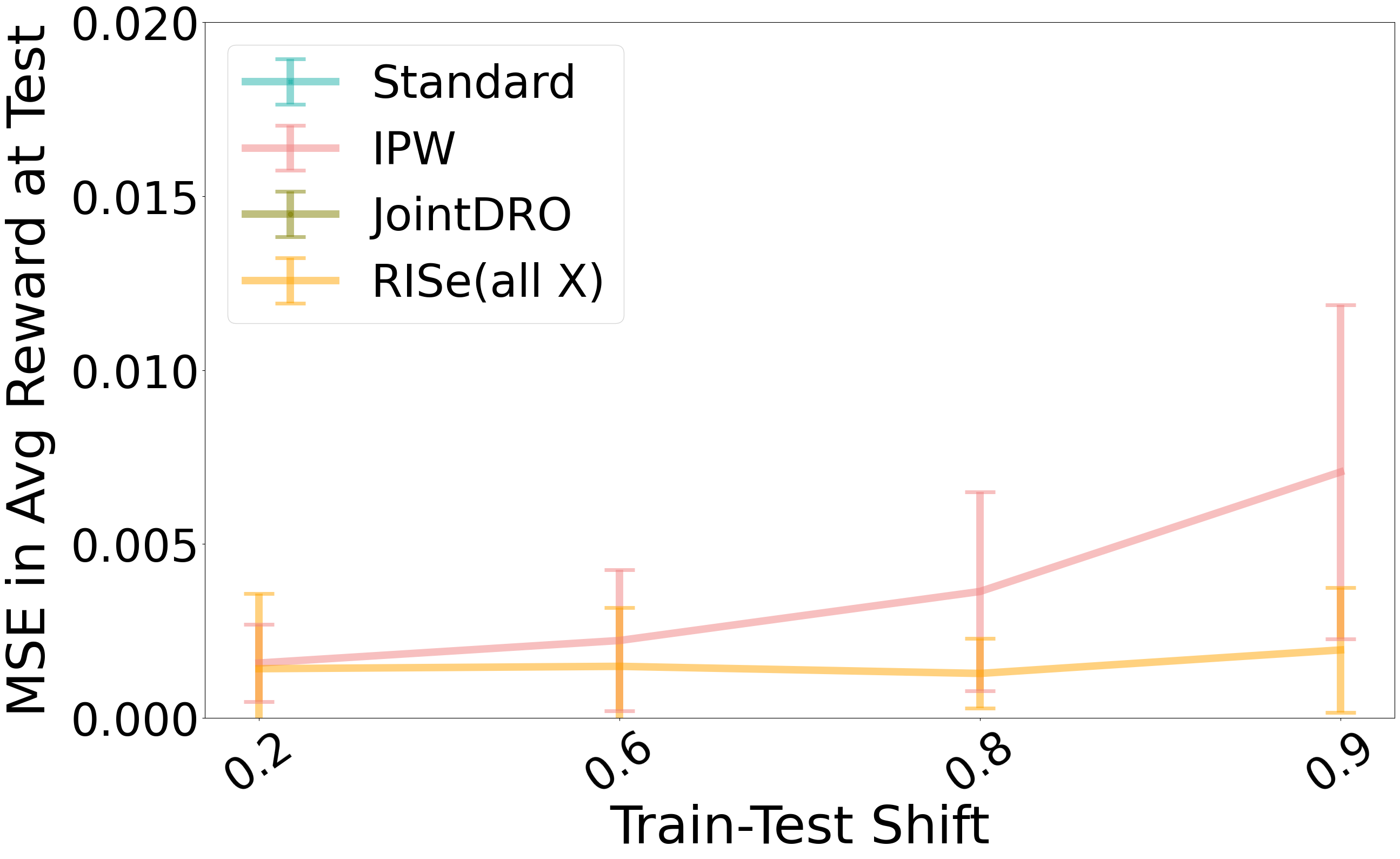}
\end{subfigure}
\caption{\textbf{Robust OPE for Warfarin dosing.} MSE in value estimate at test sets with shift in race distribution is shown. Robustness level in (\ref{eq:int_set}) is set to $\delta=0.8$. Our approach achieves the right level of conservatism to match the value at test. Mean and standard deviation for error bars are computed over $5$ random initializations. Curves for $\stdalg$ and $\divalg$ are not visible as they have high error (around 0.1 and 0.7, respectively) and lie outside the plotted y-axis range.}
\label{fig:wf_cb}
\end{figure}

\paragraph{Warfarin dosing policy.} Warfarin is an oral anticoagulant drug. Optimal dosage to assign to a patient while initiating Warfarin therapy has been a subject of multiple clinical trials \citep{heneghan2010optimal}. Using the publicly available PharmGKB dataset~\citep{international2009estimation} of $5528$ patients, \citet{bastani2015online} learn contextual bandit policies that adapt doses based on patient covariates including demographics, clinical and genetic information. 
However, the performance of the policies is suspect when applied to different patient populations than the development cohort.
We construct robust estimates of a given policy's accuracy in making correct dosing decisions under shifts in race distribution. 
We note that the optimal dose for each patient is also available in the dataset which enables evaluating bandit algorithms. Using true outcomes from a held-out dataset, we learn the optimal dosing policy with linear regression. We evaluate this policy on test data with different race distributions. Specifically, we successively subsample (without replacement) fewer patients with a recorded race into our analysis set. The policy has lower performance on the population with \texttt{Unknown} race. Thus, the value of the policy decreases with increasing shift as the relative proportion of this group increases. To estimate this value correctly, the robust method has to consider the right level of conservativeness.

Figure \ref{fig:wf_cb} shows the MSE between the estimated average reward by each method and the actual average reward evaluated on the shifted test set. We observe that $\ouralg$ achieves lower MSE than $\divalg$ and \texttt{IPW}, as it constructs the uncertainty sets for covariate shifts alone. \subsection{Robust OPE in MDP}
\label{sec:exp_fullmdp}
\paragraph{Cliffwalking domain.} We consider a $6\times6$ gridworld with start, goal positions, and a cliff on one edge \citep[][Ex. 6.6]{sutton2018reinforcement}. The agent incurs a rewards $-100$ on falling off the cliff, $-1$ for each step, and $0$ on reaching the goal. With a constant shift probability, the agent slips down one step towards the cliff instead of taking the prescribed action. This shift probability varies across environments, thus, changing the transition dynamics,  necessitating robustness in policy evaluation.
We consider two state features in the observation space. First feature is the agent's grid position, second is a discrete random noise sampled uniformly from $\{1,\dots,36\}$ with uniformly sampled rewards over [-1,0]. We evaluate the value function estimate for an agent following uniform random policy using dynamic programming with the standard Bellman equation ($\stdalg$) or the robust one in Eq. (\ref{eq:robust_bellman}) ($\divalg,\ouralg$). 

Since agent's actions affect only the first feature, $\ouralg$ correctly constructs uncertainty sets based on transition dynamics in the first feature alone, $P(s^1\vert s,a)$. In contrast, $\divalg$ ignores this structure and constructs uncertainty sets using both the features, $P(s^1,s^2\vert s,a)$. Table \ref{tab:cliff_dp} reports the value estimates for the start position, repeated over $10$ random runs. We observe that for high level of desired robustness $\delta=0.4$, $\divalg$ decreases the value from $-1136$ to $-1448$ (by 27.4\%) while $\ouralg$ only decreases it to $-1416$ (by 24.6\%). 
This provides a sanity check that both DRO methods have the expected behavior in the MDP. 

\begin{table*}[htbp!]
\centering
\begin{tabular}{p{1cm}c|cccc}
\toprule
      &    $\stdalg$ & \multicolumn{2}{c}{$\divalg$ } & \multicolumn{2}{c}{$\ouralg$} \\
      $\delta$ &          - &          0.4 &          0.8 &          0.4 &          0.8 \\
\midrule
mean & -1136.43 & -1448.16 & -1221.78 & \textbf{-1416.39} & \textbf{-1190.57} \\
std. &     6.22 &     6.32 &     5.28 &     6.70 &     6.33 \\
median & -1136.64 & -1449.91 & -1222.76 & -1417.12 & -1190.92 \\
\bottomrule
\end{tabular}
\caption{\textbf{Cliffwalking domain.} Estimated value for a uniform random policy with standard and robust dynamic programming. \ouralg~provides a less conservative value estimate than the divergence-based method \divalg~(i.e. the decrease in value is smaller). Robust methods use $\divcvar$ based sets with $\delta\in\{0.4,0.8\}$. Mean, standard error, and median reported with 10 random initializations.}
\label{tab:cliff_dp}
\begin{tabular}{p{1cm}c|cccc}
\toprule
    &   $\stdalg$    &  $\divalg$ & $\ouralg$\\
  $\delta$  &   -       &   0.8     &    0.8   \\
\midrule
mean & -0.037 & -0.939 &   \textbf{-0.662} \\
std. &  0.008 &  0.006 &    0.148 \\
median & -0.039 & -0.941 &   -0.705 \\
\bottomrule
\end{tabular}
\caption{\textbf{Sepsis domain.} Estimated value for the evaluation policy with standard and robust dynamic programming. \ouralg~provides a less conservative value estimate compared to \divalg. Robust methods use $\divcvar$ based sets with $\delta=0.8$. Mean, standard error, and median reported with 10 random initializations.}\label{tab:sepsis_dp}
\end{table*}

\paragraph{Sepsis treatment evaluation.}
This domain has a more involved transition dynamics. The Sepsis simulator~\citep{oberst19counterfactual}, has been used to test sequential treatment policies \citep{oberst19counterfactual, namkoong2020off, killian2020counterfactually,futoma2020popcorn}. The state space is discrete and includes $4$ vital signs (blood pressure, glucose levels, heart rate, and oxygen concentration) and diabetic status, with a total of 1440 states. Actions correspond to combinations of 3 treatments
(antibiotics, mechanical ventilation, and vasopressors). Terminal states are discharge from ICU or death with rewards $+1$ and $-1$, respectively. Note that, glucose levels fluctuate more for diabetics than non-diabetics. Our goal is to evaluate a learned policy, namely RL policy, obtained using policy iteration on a finite dataset sampled from an environment with $20\%$ diabetics. 

We consider a setting where the percentage of diabetics and the fluctuation in their glucose levels varies in the test environment. For interrogating the RL policy for possible deployment, we would want to find its robust value accounting for potential changes in the transition dynamics. $\divalg$ constructs uncertainty sets based on the full 1440 states, while $\ouralg$ considers uncertainty only in glucose level dynamics for diabetics and non-diabetics. Thus, $\ouralg$ represents the actual shifts more faithfully compared to $\divalg$. The goal is to check how conservative the divergence based sets are compared to the intersection ones. Table \ref{tab:sepsis_dp} reports the value estimates for the RL policy obtained with standard and robust dynamic programming. We observe that $\divalg$ reports a decrease in value by 21 times as compared to the value at train environment and $\ouralg$ only reports a decrease in value by 14 times. Thus, we demonstrate the benefit of curating the uncertainty sets using domain knowledge to balance utility and robustness. 
\section{Conclusions and Future Work}
Distribution shifts have been represented in past work either using divergence metrics or intervention shifts in the underlying causal models. 
Here, we propose to represent shifts at the intersection of the two approaches. We argue that this provides ways to represent more realistic shifts and to find robust solutions which are less conservative. Under certain conditions, the robust solutions can be obtained efficiently for such shifts. For the supervised learning case, we provide a causal interpretation of an existing method \citep{duchi2019distributionally}. We extend this to the case of OPE in contextual bandits and RL to provide a novel estimator for robust OPE.
Interesting future directions include exploring relaxations of Assumptions \ref{assum:inter}(a) by restricting the functional form of the causal relationships to identify the worst-case risk in the presence of unobserved confounding. Assumption \ref{assum:inter}(b) can be relaxed by developing techniques for solving DRO for shifts other than in marginals (e.g. conditional distributions). This will enable expressing and learning robust models for a larger and more realistic class of distribution shifts. %
 
\bibliography{robust.bib}
\bibliographystyle{plainnat}

\pagebreak

\appendix
\section{Appendix}\label{app:rslduchi}

\subsection{Proof of Proposition~\ref{sec:prop1}}\label{app:prop1}

Recall that,
\begin{equation*}
\begin{aligned}
    \Uint_P := \{&Q\ll P\ \text{s.t.}\ Q=P(V\mid do(Z\sim \nu(Z))), \\ &\infdiv{\nu(Z)}{P(Z)} \leq \delta, Z:=ch(E)\}.
\end{aligned}
\end{equation*}
\begin{proof}
By Assumption \ref{assum:inter}, we can factorize the joint distribution as, $P(\V) = P(Z)\prod_{O\in \V\setminus Z} P(O\vert pa(O))$, for any variables $Z$ with no parents. Further, the soft intervention on $Z$, denoted by $do(Z \sim \nu(Z))$, by definition, implies replacing the original marginal $P(Z)$ by $\nu(Z)$, s.t. $\infdiv{\nu(Z)}{p(Z)} \leq \delta$. Thus, the joint distribution can be written as, $P(\V) = \nu(Z)P(V\setminus Z\vert Z)$, where we collected the factors not involving $Z$ into $P(V\setminus Z\vert Z)$.
\end{proof}

\subsection{Equivalence of Marginal DRO \citep{duchi2019distributionally} and $\ouralg$ ($\Uint_{P}$) with $\divcvar$}\label{app:irslset}

In Marginal DRO \citep{duchi2019distributionally}, the uncertainty set is built in such a way that the training population contains at least $\delta$ proportion of the test population.
\begin{equation}
\label{eq:marginal_set_sub_app}
\begin{aligned}
 \mathcal{U}^\text{sub}_P := \{&Q(Z)P(V\setminus Z\vert Z)\ \text{s.t.} \\ &P(Z) = \alpha' Q(Z) + (1-\alpha') Q'(Z), \alpha \leq \alpha' \leq 1\}
\end{aligned}
\end{equation}
where $Q'(Z)$ is any distribution and $\alpha \in (0,1]$ determines the minimum size of the subpopulation shared between train and test.
Thus, $\mathcal{U}^\text{sub}_P$ constrains the ratio $\frac{Q(Z)}{P(Z)}\leq \frac{1}{\alpha}$ for all values of $Z$. 
Note that the same constraint is expressed in $\Uint_P$ when using $\infdivcvar{Q}{P} = \log{\sup_{Z\in \texttt{dom}(Z)}} \frac{Q(Z)}{P(Z)}\leq \delta$ for a suitably chosen $\delta$. Thus, the two sets are equivalent.

\subsection{More Details on Solving DRO}
The DRO problem for $\Uint_P$ solves,
\begin{equation}
\label{eq:robustint_app}
\argmin_{\theta\in \Theta} \left\{\R(\theta, \Uint_P):= \sup_{Q\in\ \Uint_P}\ \E_{\V\sim Q}[\ell(\theta, \V)]\right\}
\end{equation}
Using convex duality arguments \citep{shapiro2014lectures}, one can write the worst-case risk alternatively as,
\begin{equation}
\label{eq:robust_marginal}
\begin{aligned}
    \R(\theta, \Uint_P) &= \sup_{Q\in\ \Uint_P}\ \E_{Z\sim Q(Z)}\E_{P(V\vert Z)}[\ell(\theta,V)\vert Z] \\ &= \inf_{\eta\in \reals} \frac{1}{\delta}\E_P\left[(\E_P\left[\ell(\theta,V)\vert Z\right]-\eta)_+\right] + \eta
\end{aligned}
\end{equation}
where $(\cdot)_+ = \text{max}(\cdot, 0)$. 
Note that this requires estimating $\E_{P(V\vert Z)}[\cdot]$ which may be hard if $Z$ is continuous-valued or we do not have enough data for all possible values of $Z$. 
Assuming smoothness of this conditional loss, \citet[][Lemma 4.2]{duchi2019distributionally} gives an upper bound for the worst-case risk with the empirical version as, 

\begin{align}
\begin{split}
    \label{eq:marginal_smooth_app}
    \widehat{\R}(\theta, \Uint_P)&  \\ 
    = \inf_{\substack{\eta\geq 0,\\ B\in \reals^{n\times n}_+}}  \Bigg\{& \frac{1}{\delta} \Big(\frac{1}{n}\sum_{i=1}^n\Big(\ell(\theta, V_i) {-} \frac{1}{n}\sum_{j=1}^n(B_{ij}{-}B_{ji}) {-} \eta\Big)^2_+\Big)^{\frac{1}{2}} \\&+ \frac{L}{\epsilon n^2}\sum_{i,j=1}^n \sum_{O\in Z}\Vert O_i-O_j\Vert B_{ij} + \eta \Bigg\} \end{split}
\end{align}

for any $\epsilon>0$, where $\eta, B$ are dual variables. We use this estimator while solving the DRO problems in supervised learning and bandit experiments. The minimization is performed using gradient descent as the objective is convex in both $\eta$ and $B$.

\subsection{Justification for Using the Term Intersection}
In this result, we show that the proposed set $\Uint_P$ is indeed an intersection of the two sets, $\Udiv_P$ and $\Uintv_P$, under the class of $f$-divergence metrics which include many commonly-used metrics like KL-divergence, Total Variation, and CVaR \citep{liese2006divergences}.

Recall that,
\begin{align*}
    \Udiv_P &:= \{Q\ll P\ \text{s.t.}\ \infdiv{Q}{P}\leq \delta\},\\
    \Uintv_P &:= \{Q\ll P\ \text{s.t.}\ Q=P(V\mid do(X\sim \nu(X))), \\ &X:=ch(E)\},\\
    \Uint_P &:= \{Q\ll P\ \text{s.t.}\ Q=P(V\mid do(X\sim \nu(X))), \\ &\infdiv{\nu(X)}{P(X)} \leq \delta, X:=ch(E)\}.
\end{align*}

\begin{proposition}
Given Assumption \ref{assum:inter} holds and $\cD$ is a $f$-divergence, then $\Udiv_P \cap \Uintv_P \equiv \Uint_P$.
\end{proposition}
\begin{proof}
We can write the sets using Assumption \ref{assum:inter} as,
\begin{align*}
    \Uintv_P &:= \{Q\ll P\ \text{s.t.}\ Q=\nu(X)P(V\setminus X\mid X), \\ &\infdiv{\nu(X)}{P(X)} \rightarrow \infty\}\\
    \Uint_P &:= \{Q\ll P\ \text{s.t.}\ Q=\nu(X)P(V\setminus X\mid X), \\ &\infdiv{\nu(X)}{P(X)}\leq \delta\}.
\end{align*}

To show the equivalence of the sets, consider the two cases.

\textit{Case 1.} To prove, $Q\in \Uint_P \implies Q\in \Udiv_P\cap \Uintv_P$.

Assume $Q\in \Uint_P$.

Since $\infdiv{\nu(X)}{P(X)}\leq \delta < \infty$, this implies $Q\in\Udiv_P$ .

Also, for $Q\in \Uint_P$, we can write $\infdiv{Q}{P}$ as,
\begin{align*}
    &\infdiv{Q}{P}\\
    &= \infdiv{\nu(X)P(V\setminus X\mid X)}{P(X)P(V\setminus X\mid X)}\\
    &= \infdiv{\nu(X)}{P(X)}
\end{align*}
Here, the factor $P(V\setminus X\vert X)$ can be omitted because $\cD$ is a $f$-divergence which has the form of a ratio of densities, $\infdiv{Q}{P}=\int f\left(\frac{dQ}{dP}\right)dP$. 

Thus, the divergence between the joint distributions, $\infdiv{Q}{P} = \infdiv{\nu(X)}{P(X)} \leq \delta$. By definition, this implies $Q\in \Uintv_P$. 

Hence, $Q\in \Udiv_P\cap \Uintv_P$.

\textit{Case 2.} To prove, $Q\in \Udiv_P\cap \Uintv_P \implies \Uint_P$

Assume $Q\in \Udiv_P\cap \Uintv_P$. As $Q\in \Udiv_P$, the divergence between the joint distributions is bounded,
\begin{align*}
    &\infdiv{Q}{P}\leq \delta\\
    &\implies \infdiv{Q(X)Q(V\setminus X\mid X)}{P(X)P(V\setminus X\mid X)}\leq \delta
\end{align*}
Since, $Q\in \Uintv_P$ as well, the factors which are not intervened upon remain the same, i.e. $Q(V\setminus X\mid X)=P(V\setminus X\mid X)$. Thus, we can use the property of $f$-divergences, as earlier, to bound the divergence between the marginal distributions,
\begin{align*}
    &\infdiv{Q}{P}\leq \delta\\
    &\implies \infdiv{\nu(X)P(V\setminus X\mid X)}{P(X)P(V\setminus X\mid X)} \leq \delta\\
    &\implies \infdiv{\nu(X)}{P(X)} \leq \delta
\end{align*}
Thus, $Q\in \Uint_P$.

The two cases together imply that $\Udiv_P \cap \Uintv_P \equiv \Uint_P$.
\end{proof}
This justifies the intuition behind $\ouralg$ as an intersection of the divergence and intervention-based sets, as in Fig. \ref{fig:fig1}b.

\section{Case of Correctly Specified Models}
In the supervised learning experiment, we consider the effect of model misspecification on the quality of $\intvalg$ solution, and claim that $\ouralg$ is a better alternative in this case. Here, we investigate the behavior of the methods if, instead, the model is correctly specified. We show for a special case, where there are no variables downstream of $Y$ in the causal graph and loss function is squared loss, the two methods converge to the same solution.

For a prediction function $f(x,\theta)$, define,
\begin{align*}
\erm &:= \min_\theta\ \mathbb{E}_P[(y-f(x,\theta))^2]\\
\ouralg &:= \min_\theta\ \sup_{Q\ \in\ \Uint_P}\ \mathbb{E}_Q[(y-f(x,\theta))^2]\\
\intvalg &:= \min_\theta\ \sup_{Q\ \in\ \mathcal{U}_P^\text{int}}\ \mathbb{E}_Q[(y-f(x,\theta))^2]\\
&=\min_{\theta_c} \mathbb{E}_P[(y-f(x_c,\theta_c))^2],
\end{align*}
where $x_c$ are causal parents of $Y$ \citep{rojas2018invariant}.

By definition,
\begin{align*}
\erm &\leq \ouralg,\ \text{since}\ P\in \Uint_P,\\
\ouralg &\leq \intvalg,\ \text{since}\ \Uint_P \subseteq \mathcal{U}_P^\text{int}.
\end{align*}
If $Y$ has no descendants,
$$\erm=\intvalg = \min_{\theta_c} \mathbb{E}_P[(y-f(x_c,\theta_c)^2],$$
since $x_c$ is Markov blanket of $Y$.
But, as shown earlier, $$\erm\leq\ouralg\leq\intvalg.$$
Thus, $\ouralg=\intvalg$. 

This means that when models are correctly specified $\ouralg$ and $\intvalg$ have the same performance. If the model is misspecified, $\ouralg$ can outperform $\intvalg$ significantly, as observed in the experimental results in Sec \ref{sec:exp_superlearn}.
This adds support for using $\ouralg$ instead of $\intvalg$ when we suspect that models are misspecified. However, a detailed study is required to compare the two methods under different amounts of model misspecification and training data.

\section{Algorithm for Robust OPE for CBs}\label{app:cb_ope}
Procedure for estimating robust average reward in case of contextual bandits is given in Algorithm \ref{alg:cb_ope}.

\begin{algorithm}[!htbp]
\caption{Robust OPE for CBs}
\label{alg:cb_ope}
\begin{algorithmic}
\STATE {\bfseries Input:} Data $\{Z_{i}, T_i, Y_i\}_{i}$, Target policy $Q(T|Z)$, behavior policy $P(T|Z)$, hyperparameters $\delta$, \texttt{L}, \texttt{lr}.
\STATE
\STATE Compute importance weights $\{W_i=\frac{Q(T_i|Z_i)}{P(T_i|Z_i)}\}_i$.
\STATE Create dataset $\{V_i = (Z_{i}, W_i\times Y_i)\}_{i}$.
\STATE Estimate $\widehat{R}(\cU^{\text{CB}}_P)$ (Eq.~(\ref{eq:robust_eval})) with the worst-case risk estimator in Eq. (\ref{eq:marginal_smooth_app}) for the dataset.
\STATE
\STATE{{\bf{return}} $\widehat{R}(\cU^{\text{CB}}_P)$}
\end{algorithmic}
\end{algorithm}

\section{Algorithm for Robust OPE for MDPs}
\label{app:algmdp}
Robust OPE with dynamic programming amounts to solving the following fixed point equation iteratively,
\begin{equation}
    \begin{aligned}
        V^\pi(s) = r(s,\pi(s)) + \inf_{P\in\U^{\text{MDP}}(s,\pi(s))} \gamma \E_{s'\sim P(s'\vert s,\pi(s))}[V^\pi(s')]
\end{aligned}
\end{equation}

At each iteration, we have to solve the DRO problem,
\begin{align*}
    &\inf_{P\in\mathcal{U}^{\text{MDP}}(s,\pi(s))} \E_P[V^\pi(s')]\\
    &= \inf_{P\in\mathcal{U}^{\text{MDP}}(s,\pi(s))} \E_{P(s'^{1}\vert s, \pi(s))} \left[\E_{P_0(s'^{2}\vert s, \pi(s), s'^{1})}\left[V^\pi(s')\right]\right]\\
    &=: \cR_{\mathcal{U}^{\text{MDP}}(s, \pi(s))} \left(V^\pi(s')\right)
\end{align*}
We estimate the inner expectation w.r.t. $P_0$ with Monte-Carlo averaging on batch data as it contains samples from $P_0$. We then compute the value function update using importance sampling with data collected from the policy $\mu$,
\begin{align}
    V^\pi(s) &\leftarrow \frac{\pi(a\vert s)}{\mu(a\vert s)} \left[r(s,a) + \gamma \cR_{\mathcal{U}^{\text{MDP}}(s, a)} \left(V^\pi(s')\right)\right],\nonumber\\
    &a\sim\mu(a\vert s), s'\sim P_0(s'\vert s, a)
\label{eq:robust_mdp_app}
\end{align}
As earlier, we choose the divergence metric $\divcvar$, and accordingly, solve the DRO problem using the form in Eq. (\ref{eq:robust_marginal}). This requires access to the true transition model, $P_0$, and the reward model, $r$. Since we only have samples from $P_0$, we use the maximum likelihood estimate $\widehat{P_0}$ from the batch data in place of $P_0$. For finite state-action space, this amounts to counting transitions for each tuple $(s,a,s')$ and averaging. We assume that the reward model is known.
The procedure is given in detail in Algorithm \ref{alg:cb_mdp}.

\begin{algorithm}[!htbp]
\caption{Robust OPE for MDPs}
\label{alg:cb_mdp}
\begin{algorithmic}
\STATE {\bfseries Input:} Trajectories $\{(s,a,s',r)\}$ sampled using policy $\mu$, Target policy $\pi$, Discount factor $\gamma$, Robustness level $\delta$.
\STATE
\STATE Learn transition model $\widehat{P_0}(s'^1\vert s,a)= \frac{\text{Count}\{(s,a,(s'^1,*))\}}{\text{Count}\{(s,a,(*,*))\}}$ and $\widehat{P_0}(s'^2\vert s,a,s'^1)= \frac{\text{Count}\{(s,a,(s'^1,s'^2))\}}{\text{Count}\{(s,a,(s'^1,*))\}}$.
\STATE Initialize $V^\pi(s)=0$, for all $s\in\cS$.
\REPEAT
\FOR{$s\in\cS$}
\STATE Update $V^\pi(s)$ using Eq. (\ref{eq:robust_mdp_app}) with $\widehat{P_0}$.
\ENDFOR
\UNTIL{$V^\pi$ converges}
\STATE
\STATE{{\bf{return}} $V^\pi$}
\end{algorithmic}
\end{algorithm}

\textbf{Extension to continuous states} Although, the method is specified for finite state MDPs, it can be extended to continuous or large state spaces by leveraging linear function approximations along with robust dynamic programming, as proposed in \citep{tamar2014scaling}. Validating such approximations for $\ouralg$ is a fruitful direction for more work.

\section{Experimental Details}\label{app:expts_sl}

For the supervised learning experiment in Sec. \ref{sec:exp_superlearn}, ranges for hyperparameters are: $\texttt{lr}\in\{0.1,0.2,\dots,0.8\}, \texttt{L}\in \{0.1,1,10,100,1000\}$, and $\delta\in\{0.001,0.01,0.1\}$.

For the supervised learning and CB experiments, gradient descent is performed with Adagrad \citep{duchi11adagrad} implementation in PyTorch \citep{paszke19pytorch}. A preprocessed version of the Warfarin dataset was downloaded from \url{https://github.com/khashayarkhv/contextual-bandits/blob/master/datasets/warfarin.csv}.

For the MDP experiments, convergence condition for $V^\pi$ is taken as maximum change in $V^\pi(\cdot) < 10^{-2}$ after one repetition across all states. The minimization over $\eta$ while solving Eq. (\ref{eq:robust_marginal}) is performed using Brent's method, implemented in the package \texttt{scipy} \citep{scipy}.

Experiments were run on a compute cluster using a single node with a 2.50 GHz Intel processor, 2 GPUs with 6 GB memory each, and 256 GB system memory.

\begin{figure}[htbp!]
\centering
\includegraphics[width=0.7\linewidth]{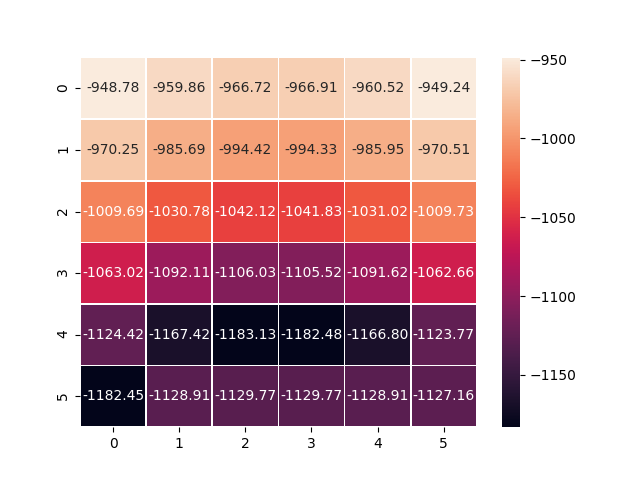}
\caption{\textbf{Illustration of the Cliffwalking domain.} Plot shows (part of) the value function estimated using robust Bellman equation with \ouralg. Start position is (5,0), goal is (5,5), and cliff corresponds to the row (5,0) to (5,5). Agent slips downward by 1 cell with probability 0.1 when taking actions in any of the columns except first and last columns. Results in Table \ref{tab:cliff_dp} report the value at the start position (5,0), which is -1182.45 here, averaged over 10 random initializations of the domain.}
\label{fig:cliff}
\end{figure}

\section{More Details on Cliffwalking Domain}
We consider a $6\times6$ gridworld (Figure \ref{fig:cliff}) with start, goal positions, and a cliff on one edge \citep[][Ex. 6.6]{sutton2018reinforcement}. The agent incurs a rewards $-100$ on falling off the cliff, $-1$ for each step, and $0$ on reaching the goal.
To add stochasticity in transitions, we make two changes to the domain. With a constant shift probability, the agent slips down one step towards the cliff instead of taking the prescribed action. This shift probability varies across environments, thus, changing the transition dynamics, necessitating robustness in policy evaluation.
Second, we consider two state features in the observation space of the agent. The first feature is the agent's position on the grid, and the second is discrete random noise sampled uniformly from the set $\{1,2,\dots,6\times6\}$. The second feature has an associated reward sampled from Uniform$(-1,0)$ for the $6\times 6$ values it takes. Total reward is the sum of rewards from the two features (one based on grid position and one sampled uniformly between -1,0). 

Since agent's actions affect only the first feature, $\ouralg$ correctly constructs uncertainty sets based on transition dynamics in the first feature alone, $P(s^1\vert s,a)$. In contrast, $\divalg$ ignores this structure and constructs uncertainty sets using both the features, $P((s^1,s^2)\vert s,a)$.
We evaluate the value function estimate for an agent following uniform random policy using dynamic programming with the standard Bellman equation ($\stdalg$) or the robust one in Eq. (\ref{eq:robust_bellman}) ($\divalg,\ouralg$).

\section{More Details on Sepsis Domain}
The RL policy to be evaluated is obtained as follows. The optimal policy for the Sepsis simulator, namely the physician policy, is found with the procedure used in \citep{oberst19counterfactual}. Physician policy is made stochastic by taking random action with probability 0.05. With this policy, a dataset with $1000$ trajectories each with maximum length of $20$ is sampled. Diabetic population is fixed to 20\% while sampling. The RL policy used in evaluation is obtained by running policy iteration on this dataset. Then, RL policy is used to sample another dataset with $10000$ trajectories each with maximum length of $20$, and 20\% diabetics. This data is used for the final evaluation. As the policy being evaluated is the same as the one used to sample the available data, the importance weighting done in Eq. (\ref{eq:robust_mdp_app}) is not needed.

\end{document}